\documentclass[sigconf, nonacm]{acmart} 

\newif\ifincludeappendix
\includeappendixtrue

\acmConference[ICAIF '26]{6th ACM International Conference on AI in Finance}{%
  November 2026}{}
\acmYear{2026}
\copyrightyear{2026}
\acmBooktitle{Proceedings of the 6th ACM International Conference on AI in
  Finance (ICAIF '26)}

\setcopyright{none}

\usepackage{amsmath, amsthm}
\usepackage{mathtools}
\usepackage{dsfont} 
\usepackage{booktabs}
\usepackage{multirow}
\usepackage{algorithm}
\usepackage{dsfont} 
\usepackage{algpseudocode}
\usepackage{enumitem}
\usepackage{tikz}
\usepackage{pgfplots}
\usepackage{subcaption}
\pgfplotsset{compat=1.18}
\usepgfplotslibrary{statistics}
\usetikzlibrary{positioning, arrows.meta, shapes.geometric, fit, calc}

\newtheorem{theorem}{Theorem}
\newtheorem{lemma}{Lemma}
\newtheorem{proposition}{Proposition}

\newtheorem{example}{Example}

\newcommand{\R}{\mathbb{R}}

\DeclareMathOperator*{\argmax}{arg\,max}

\title{
 Reinforcement Learning with Hawkes-Structured Networks for Market Making}

\author{Fernando Acero}
\authornote{All authors contributed equally. \\ \textbf{Email:} fernando.acero@jpmorgan.com, \{p.plank24, yufei.zhang\}@imperial.ac.uk}
\affiliation{%
  \institution{J.P. Morgan QTR AI Research}
  \country{United Kingdom}
}

\author{Philipp Plank}
\authornotemark[1] 
\affiliation{%
  \department{Department of Mathematics}
  \institution{Imperial College London}
  \country{United Kingdom}
}

\author{Yufei Zhang}
\authornotemark[1] 
\affiliation{%
  \department{Department of Mathematics}
  \institution{Imperial College London}
  \country{United Kingdom}
}

\begin{document}

\begin{abstract}

Market making under realistic order-flow dynamics is challenging due to the clustered, self-exciting nature of order arrivals, which are naturally modeled by non-Markovian Hawkes processes. We propose a model-free reinforcement learning (RL) framework for market making with general Hawkes order-flow dynamics. We first establish non-asymptotic performance bounds that quantify the loss incurred by approximating the infinite-dimensional Hawkes state using a finite collection of exponentially weighted features. Guided by this analysis, we introduce the Hawkes-Structured Network (HSN), a novel neural architecture that encodes historical order flow through a learnable mixture-of-exponentials filter. HSN is trained end-to-end using model-free RL, jointly optimizing the feature representation and quoting policy directly from reward feedback. Experiments on single- and multi-asset market-making problems with mixture-of-exponentials, power-law, and regime-switching kernels demonstrate that HSN outperforms general-purpose neural architectures, achieving higher trading rewards while implicitly recovering the underlying order-flow intensities.
\end{abstract}

\ccsdesc[500]{Computing methodologies~Reinforcement learning}
\ccsdesc[300]{Theory of computation~Stochastic control and optimization}
\ccsdesc[300]{Computing methodologies~Neural networks}

\keywords{Market making, reinforcement learning, Hawkes process, power-law kernel, 
representation learning, Hawkes-Structured Network}

\maketitle

\section{Introduction}
\label{sec:intro}

Market making is a fundamental mechanism of modern electronic financial markets. A market maker (MM) continuously posts bid and ask quotes in a limit order book, earning the bid–ask spread while bearing inventory risk from the directional exposure accumulated through filled orders. The problem of optimally setting these quotes has been extensively studied within a stochastic control framework, in which order arrivals are modeled as given (controlled) jump processes, and optimal quoting rules are derived by solving the associated Hamilton–Jacobi–Bellman   (HJB) equations; see, e.g., \text{\cite{avellaneda_stoikov_2008, cartea_jaimungal_ricci_2014, cartea_jaimungal_penalva_2015, gueant_2013, bergault2021closed, barucci2025market, barzykin2025fx}}. While these models have shaped both the academic literature and  industrial practice, they typically rely on restrictive assumptions, such as   memoryless Poisson arrivals, low-dimensional Markovian representations of order-flow dynamics, and fully specified arrival processes and fill probabilities. 
 
A substantial body of empirical evidence shows that high-frequency order-flow data exhibit pronounced clustering, self-excitation, and cross-excitation across order types and assets. 
Hawkes processes provide a parsimonious framework for modeling these phenomena and have been widely applied to characterize  order book dynamics 
\cite{bowsher2007modelling, filimonov2012quantifying, bacry_2013, hardiman_2013, bacry_2015a, morariu_2021, jain2024limit}. Importantly, empirical studies consistently indicate that the excitation kernels governing these processes often exhibit slow, power-law decay rather than simple exponential decay \cite{filimonov2012quantifying, hardiman_2013, jain2024limit}. This renders the associated market-making problem intrinsically infinite-dimensional and non-Markovian, as the relevant state depends on the entire history of order events.

Recent work has begun to address Hawkes-driven market making beyond exponential kernels, but still in a model-based setting with given excitation kernels and fill probabilities. For instance, 
\cite{lai2023data}  first estimates a power-law kernel and then solves the market-making problem by   deep learning using the estimated intensity. Similarly, {\cite{martin_tse_2026}} constructs a high-dimensional Markovian approximation of a given   kernel and solves the corresponding   HJB equations using deep learning. This two-stage model-based paradigm, in which one first specifies or estimates a Hawkes model for the arrival dynamics (system identification) and then optimizes a policy under the approximate kernel (model-based control), is  sensitive to kernel misspecification. In realistic  environments, where the excitation kernel or the fill probability function may undergo regime changes, such a model-based approach  cannot adapt effectively and  may suffer significant performance degradation over time.

\paragraph{Our work.}

This paper develops a model-free reinforcement learning (RL) framework for market making under non-Markovian Hawkes order-flow dynamics. Market orders are generated by   Hawkes processes with general excitation kernels, while the MM optimizes bid and ask quotes without prior knowledge of the kernel or fill probabilities. 
We solve the non-Markovian RL problem by   learning effective quoting policies with \emph{Hawkes-structured neural networks}, combining the interpretability of model-based control with the robustness of model-free RL. 
Our contributions are threefold. 
\begin{itemize}[wide]
\item We derive in Section \ref{sec:regret} non-asymptotic performance guarantees for Hawkes-driven market making when policies depend on a finite history of events filtered through a finite collection of exponentially weighted features; see Theorem \ref{thm:disc_trunc_kernel_reward_bnd}. The analysis refines existing Markovian kernel approximation results and quantifies how kernel approximation error translates into value function loss.
\item Motivated by the theoretical analysis, we introduce in Section~\ref{sec:hsn} the Hawkes-Structured Network (HSN), a novel neural architecture whose feature extractor realizes a finite-dimensional Markovian approximation of the Hawkes intensity through a learnable mixture-of-exponentials filter. Trained \emph{end-to-end} with model-free RL, HSN learns task-oriented Markovian representations directly from data, yielding improved computational efficiency  and interpretability over black-box neural networks; see Section \ref{sec:advantage}.
\item We show in Section \ref{sec:experiments} that HSN outperforms a range of architecture baselines, including multilayer perceptrons, recurrent neural networks, and Transformers, across single-asset and multi-asset Hawkes environments with exponential, fractional, and regime-switching kernels, achieving both higher   rewards and substantially more accurate reconstruction of the underlying intensities. 
\end{itemize}

\section{Related Work}

\paragraph{RL  for market making}

 Most existing works on RL for market making assume that the order flow is Markovian,  such as  a memoryless Poisson process (see, e.g., \text{\cite{spooner_2018, ganesh2019reinforcement, gueant2019deep, vicente2021deep, cao2026logarithmic}}) or  Hawkes processes with exponential kernels 
 \cite{gavsperov2022deep}; see  \cite[Section 4.4]{hambly2023recent} for a detailed review.
  RL for market making under more realistic non-Markovian Hawkes models   is relatively rare, and has largely focused on model-based, estimation‑then‑optimization approaches \cite{lai2023data, martin_tse_2026}, which first   estimate a Hawkes model for the order dynamics and then optimize  under the model. Such methods are sensitive to kernel misspecification and nonstationary environments.

Our approach is a model-free counterpart to this line of work. We use the Markovian lift idea to motivate the HSN architecture with exponentially filtered histories, but train the policy end-to-end from reward feedback. 
To the best of our knowledge, the only other model-free RL work that explicitly uses non-Markovian Hawkes dynamics is  \cite{jain_2025}, where the agent learns when to intervene rather than continuously quoting bid and ask spreads, and the policy neither exploits Hawkes structure nor provides performance guarantees.

 \paragraph{Markovian approximation of path-dependent problems in finance}

 It is well understood that Hawkes processes (and, more generally, Volterra processes) with general kernels admit infinite-dimensional Markovian lifts; see, e.g., \cite{bonesini2023rough, cuchiero_teichmann_2020}. The error incurred by finite-dimensional truncations has been analyzed in \cite{bayer2023markovian, alfonsi2024approximation}. However, these works focus primarily on uncontrolled dynamics. We extend this line of research to sequential decision-making problems by quantifying how the approximation error propagates to control performance, yielding an explicit regret bound for optimizing policies represented by Hawkes-structured neural networks.
\section{Main  Theoretical Results}
\label{sec:setup}

We formulate a market-making problem in which market orders are driven by a non-Markovian Hawkes process, propose a class of exponentially filtered policies based on truncated histories, and establish performance guarantees for these policies.
For the sake of presentation, we focus on the single-asset setting. Similar results extend to the multi-asset case. 

\subsection{Hawkes-driven market-making problems}
\label{sec:problem}

The framework extends the classical Avellaneda-Stoikov market-making model \cite{avellaneda_stoikov_2008}
with Hawkes-driven market orders. Fix a finite horizon $T>0$, and let
$(\Omega,\mathcal F,(\mathcal F_t)_{t\in[0,T]},\mathbb P)$
be a filtered probability space supporting all stochastic processes defined later.

Let $M^+$ and $M^-$ denote two independent Hawkes processes representing the cumulative numbers of bid and ask market orders, respectively. Their intensities satisfy for all 
$c\in\{+,-\}$,
\[
\lambda_t^c
=
\lambda_0^c(t) 
+
\int_{[0,t)}
g^c(t-s)\,dM_s^c,
\]
where 
$\lambda_0^c:[0,\infty)\to [0,\Lambda_c]$ is 
 a   measurable  baseline intensity, and  $g^c: (0,\infty) \to \mathbb{R}^{\ge 0}$  is an excitation kernel. 
Following \cite[Theorem~1.4]{schilling2012bernstein}, we assume $g^c$ admits the Bernstein representation
\begin{equation}
\label{eq:bernstein_representation}
g^c(t)
=
\int_0^\infty
e^{-\beta t}\,
\vartheta^c(d\beta), \quad t>0,
\end{equation}
where $\vartheta^c$ is a measure on $[0,\infty)$, and   satisfies 
$g^c\in L^2([0,T])$
and
$\|g^c\|_{L^1([0,T])}<1$.
Condition \eqref{eq:bernstein_representation}
holds for all completely 
monotone kernels, 
including the following examples.
 
\begin{example}
 \label{eg:kernel}
For the mixture-of-exponentials kernel $g(t) = \sum_{j=1}^{p_0} \alpha_j e^{-\beta_j t}$ with $\alpha_j, \beta_j > 0$, one has $\vartheta(d\beta) = \sum_{j=1}^{p_0} \alpha_j \delta_{\beta_j}(d\beta)$, where $\delta_{\beta_j}$ denotes the Dirac measure at $\beta_j$. For the shifted power-law kernel $g(t) = \gamma (t + \tau_0)^{-(1+\alpha)}$ with $\alpha, \gamma, \tau_0 > 0$, one has $\vartheta(d\beta) = \gamma \beta^\alpha e^{-\beta \tau_0}\, d\beta / \Gamma(1+\alpha)$, where $\Gamma$ is the Gamma function\footnote{\cite{hardiman_2013} reports $\alpha \approx 0.15$ for the S\&P futures contract}.
\end{example}

Given $M^\pm$, the midprice process evolves according to $S_t = S_0 + \sigma W_t + \epsilon^+ M_t^+ - \epsilon^- M_t^-$, where $W$ is a standard Brownian motion independent of $M^\pm$, and $S_0, \sigma, \epsilon^\pm \ge 0$ are constants. Each buy (respectively sell) order permanently shifts the midprice upward by $\epsilon^+$ (respectively downward by $\epsilon^-$).

At each time $t$, the market maker (MM) posts predictable bid and ask spreads $(\delta_t^-, \delta_t^+) \in [\underline\delta, \bar\delta]^2$ with $\bar\delta > \underline\delta \ge 0$ fixed, yielding bid quote $S_{t^-} - \delta_t^-$ and ask quote $S_{t^-} + \delta_t^+$. Conditional on an arrival on side $c \in \{+,-\}$, the order is executed with probability $p_t^c = p^c(\delta_t^c)$, where $p^c: [\underline\delta, \bar\delta] \to [0,1]$ is the fill probability function. Let $N^+$ and $N^-$ denote the cumulative numbers of executed sell and buy trades. The inventory and cash processes follow $Q_t = q_0 + N_t^- - N_t^+$ and $X_t = x_0 + \int_0^t (S_{s^-} + \delta_s^+)\, dN_s^+ - \int_0^t (S_{s^-} - \delta_s^-)\, dN_s^-$, with $q_0 \in \mathbb Z$ and $x_0 \in \mathbb R$ the initial inventory and cash balance.

The MM maximizes expected terminal cash subject to inventory risk over a policy class $\Pi$. Let $\tau_1 < \tau_2 < \cdots$ denote the successive arrival times of $M^+ \cup M^-$, with $c_j \in \{+,-\}$ indicating the side of the order at $\tau_j$. The policy class consists of measurable functions $(\Phi_j^+, \Phi_j^-)_{j \ge 0}$ with $\Phi_j^c: [0,T]^j \times \{+,-\}^j \times \mathbb Z \to [\underline\delta, \bar\delta]$ such that $\delta_t^c = \Phi_j^c((\tau_k, c_k)_{k \le j}, Q_{\tau_j})$ for $t \in (\tau_j, \tau_{j+1}]$ and $\delta_t^c = \Phi_0^c(q_0)$ for $t \in [0, \tau_1]$. For $\kappa, \gamma > 0$, the MM seeks to maximize
\begin{equation}
\label{eq:objective}
    J^g(\pi) = \mathbb E[R_\pi^g],
    \quad
    R_\pi^g \coloneqq X_T + Q_T(S_T - \kappa Q_T) - \gamma \int_0^T Q_t^2\, dt,
\end{equation}
over $\pi \in \Pi$, where $\kappa$ and $\gamma$ are penalty parameters on the terminal and running inventory. The superscript $g$ in \eqref{eq:objective} indicates the dependence of the objective on the Hawkes kernel $g = (g^+, g^-)$.

Since the order-flow processes $M^\pm$ are non-Markovian for general excitation kernels, solving \eqref{eq:objective} is computationally challenging even under complete knowledge of $g^\pm$ and $p^\pm$. We take one step further and consider the more realistic setting in which these quantities are unknown, and develop a model-free RL approach to learn near-optimal quoting strategies directly from data.

\subsection{Regret with exponentially filtered policies}
\label{sec:regret}

Directly searching for an optimal policy over $\Pi$ is computationally challenging, since the quoting strategy depends nonlinearly on the entire order-flow history. We therefore restrict attention to a class of exponentially filtered policies, in which the policy depends on the history only through a finite collection of exponentially weighted statistics, and establish bounds on the resulting performance gap. This policy class is motivated by the Bernstein representation \eqref{eq:bernstein_representation}: approximating the integral by a numerical quadrature yields a finite mixture-of-exponentials kernel for $g^c$, under which the market-making problem admits a finite-dimensional Markovian formulation. 

\subsubsection{Kernel truncation}
For each $g\in \{g^+,g^-\}$, $0  \le K_1 < K_2 < \infty$ and $p\in \mathbb N$, consider the following approximate kernel 
\begin{equation}
    \label{eq:kernel_approx_L_p}
    \Hat{g}^{K_1,K_2,p}(t) \coloneqq \sum_{j = 1}^p \vartheta( I_j^{K_1,K_2,p} ) e^{-\beta_j^{K_1,K_2,p} t},
\end{equation}
where
$ 
I_j^{K_1,K_2,p}
=
\left[
K_1+ {(j-1)(K_2-K_1)}/{p},
\,
K_1+{j(K_2-K_1)}/{p}
\right),
$ 
and $\beta_j^{K_1,K_2,p}\in I_j^{K_1,K_2,p}$.
The following proposition quantifies the error $\Hat{g}^{K_1,K_2,p}  -g$ in terms of  $K_1,K_2,p$ and the   tail kernels
$ 
g_l^{K_1}(t)
\coloneqq
 \int_{0}^{K_1}
e^{-\beta t}\,\vartheta(d\beta)$ and $
g_u^{K_2}(t)
\coloneqq
\int_{K_2}^{\infty}
e^{-\beta t}\,\vartheta(d\beta)$, similar to  \cite{alfonsi2024approximation}.

\begin{proposition}
     \label{prop:L2_approx}
    For all $g\in \{g^+,g^-\}$,
    \begin{equation*}
        \lVert g - \Hat{g}^{K_1,K_2,p} \rVert_{L^2([0,T])} \le  {\varepsilon_l(K_1)} + {\varepsilon_u(K_2)} + \frac{ T^{3/2} (K_2 - K_1)\vartheta([K_1,K_2))}{\sqrt{3} p},
    \end{equation*}
where 
$ \varepsilon_l(K_1)
\coloneqq
\|g_l^{K_1}\|_{L^2([0,T])}$ and $
\varepsilon_u(K_2)
\coloneqq
\|g_u^{K_2}\|_{L^2([0,T])}.
$

\end{proposition}

\begin{proof}
Define $g^{K_1,K_2}(t)
\coloneqq
\int_{K_1}^{K_2}
e^{-\beta t}\,\vartheta(d\beta)$.
Then 
    \begin{equation*}
        \begin{split}
            &\lVert g - g^{K_1,K_2} \rVert_{L^2([0,T])}  \le  \left\lVert  g_l^{K_1} \right\rVert_{L^2([0,T])}   +\left\lVert g_u^{K_2}  \right\rVert_{L^2([0,T])} .  
        \end{split}
    \end{equation*}
 A similar argument as that for  \cite[Corollary 3.1]{alfonsi2024approximation} gives for all $t \in [0,T]$, $\lvert \Hat{g}^{K_1,K_2,p}(t) - g^{K_1,K_2}(t) \rvert \le p^{-1} t (K_2 - K_1) \vartheta([K_1, K_2))$. Hence, $\lVert \Hat{g}^{K_1,K_2,p} - g^{K_1,K_2} \rVert_{L^2([0,T])} \le p^{-1} 3^{-1/2} T^{3/2} (K_2 - K_1) \vartheta([K_1,K_2))$.  
\end{proof}  

The bound in Proposition \ref{prop:L2_approx} specializes to explicit rates for concrete kernels. The mixture-of-exponentials approximation \eqref{eq:kernel_approx_L_p} is exact for the kernel of the same type in Example \ref{eg:kernel} once $K_1$ is sufficiently small and $K_2, p$ sufficiently large; for the shifted power-law kernel, direct computation using the corresponding measure $\vartheta$ yields the asymptotics of $\varepsilon_l(K_1)$ as $K_1 \to 0$ and $\varepsilon_u(K_2)$ as $K_2 \to \infty$.

\subsubsection{Regret   of greedy strategies based on truncated kernels}

We then study the performance loss incurred by deploying a policy optimized under the approximate kernels in the true model. This  regret analysis provides the theoretical motivation for the exponentially filtered policies proposed later, which are learned directly via model-free RL under the true kernel; see Sections \ref{label:insight} and \ref{sec:hsn}.

Specifically, let $\hat{g}=\hat{g}^
\pm$ be the  approximation of the true kernels $g=g^\pm$ given in \eqref{eq:kernel_approx_L_p}, and let
$ 
\hat{\pi}\in\argmax_{\pi\in\Pi} J^{\hat{g}}(\pi)
$  
be an optimal policy of  the associated market-making problem. The following theorem  quantifies  the   optimality gap incurred by deploying $\hat{\pi}$ in the true environment, referred to as the \emph{regret}, which is defined by 
 \[
 \operatorname{Reg}(\hat \pi)\coloneqq 
\sup_{\pi\in\Pi} J^g(\pi)-J^g(\hat{\pi}),
\] 
where the performance is evaluated using the true kernel $g$ in \eqref{eq:bernstein_representation}. The proof is deferred to Section \ref{sec:proof}.


 \begin{theorem}
    \label{thm:disc_trunc_kernel_reward_bnd}
Let 
$0\le K_1<K_2<\infty$ and 
$p\in \mathbb N$. 
Suppose that there exists $\rho^*\in [0,1)$ such that 
$  \lVert \Hat{g}^{\pm,K_{1}, K_{2}, p} \rVert_{L^1([0,T])} \le \rho^*$. Then 
$$
 \operatorname{Reg}(\hat \pi)
 \le C\sum_{c\in \{+,-\}}
\left({\varepsilon_l^c(K_1)} + {\varepsilon_u^c(K_2)} + \frac{ T^{3/2} (K_2 - K_1)\vartheta^c([K_1,K_2))}{\sqrt{3} p}\right)^{1/2}, 
$$
where 
 $C>0$   depends on model parameters and $\rho^*$, but is independent of $K_1, K_2 $ and $p$.

\end{theorem}

Theorem~\ref{thm:disc_trunc_kernel_reward_bnd} quantifies  the  performance gap of 
$\hat \pi$ explicitly in terms of $K_1, K_2$ and $p$,
under  the stability condition $  \lVert \Hat{g}^{\pm,K_{1}, K_{2}, p} \rVert_{L^1([0,T])}<1$.
  By Proposition~\ref{prop:L2_approx}, this condition is guaranteed whenever   $\vartheta^\pm(\{0\})=0$, $K_1$ is sufficiently small, and 
$K_2, p$ are sufficiently large. 
 
 \subsubsection{Insights from Theorem~\ref{thm:disc_trunc_kernel_reward_bnd} for policy design}
\label{label:insight}

Theorem~\ref{thm:disc_trunc_kernel_reward_bnd} motivates optimizing \eqref{eq:objective} over policies whose decisions at time $t$ depend on historical order flow via the exponentially weighted features
\begin{equation}
\label{eq:intensity_feature}
  h^c_j(t) = \alpha^c_j \sum_{i:\, \tau_i \le t,\, c_i = c} e^{-\beta^c_j (t - \tau_i)},
    \quad j = 1, \ldots, p,
\end{equation}
for parameters $(\alpha^c_j, \beta^c_j) \in [0, \infty)^2$. The approximate control problem obtained by replacing $g^c$ with the truncated kernel \eqref{eq:kernel_approx_L_p} is Markovian in the augmented state $(\hat h^c_j)_{j=1,\ldots,p}$ with $\hat\alpha_j^c = \vartheta^c(I_j^{K_1,K_2,p})$ and $\hat\beta_j^c = \beta_j^{c,K_1,K_2,p}$, so its optimal policy takes the processes \eqref{eq:intensity_feature} as feedback variables. By Theorem~\ref{thm:disc_trunc_kernel_reward_bnd}, this policy is nearly optimal for the true problem once $p$ is sufficiently large.

The sum in \eqref{eq:intensity_feature} can be restricted in practice to the most recent $L$ events without material loss of information, because $\beta_j^c \ge K_1 > 0$ makes the contribution of a past arrival to $h^c_j(t)$ decay exponentially with the elapsed time. We exploit this observation in Section~\ref{sec:hsn} when constructing the feature representation used by our RL policy.

\section{End-to-End  Learning via Hawkes-Structured Networks}
\label{sec:hsn}

Motivated by the discussions   in Section \ref{label:insight}, this section proposes the Hawkes-Structured Network (HSN), a neural architecture whose feature extractor is precisely the finite-dimensional Markov state associated with mixture-of-exponential Hawkes kernels. This HSN is designed to solve the non-Markovian market-making problem \eqref{eq:objective} by embedding the underlying Hawkes structure directly into the network architecture. Compared with generic neural network architectures that learn historical dependencies from raw order-flow data, the proposed HSN is tailored to Hawkes-driven environments, leading to improved data efficiency and more interpretable decision-making through its structured feature representation (see Section \ref{sec:experiments}). Moreover, the regret analysis in Theorem~\ref{thm:disc_trunc_kernel_reward_bnd} applies to policies learned on top of these filtered features up to RL training error.

\subsection{Design of  network architectures}

The policy and value networks share an exponentially weighted feature extractor that summarizes the order-flow history, followed by separate multilayer perceptrons (MLPs) that map the extracted features to actions and value estimates.

For brevity, we only present the  HSN architecture of the policy network, as the value network has a similar architecture. For a $C$-channel Hawkes process,
where $C=2m$ in an $m$-asset environment (representing the bid and ask order flows for each asset),
 HSN utilizes a learnable $C \times p$-dimensional filter defined via the exponentially weighted features (cf.~ \eqref{eq:intensity_feature}):
 for all $l=1,\ldots, m$, $c\in \{+,-\}$, and $j=1,\ldots, p$, 
\begin{equation}
\label{eq:policy_exponential}
  h^{c,l}_j(t) \;=\; \alpha^{c,l}_j \sum_{i:\, \tau_i^l \le t,\, c_i = c}
    \exp\big(-\beta^{c,l}_j (t - \tau_i^l)\big),
\end{equation}
where $(\alpha^{c,l}_j, \beta^{c,l}_j)\in(0,\infty)^2$ are   learnable parameters and the summation is restricted to the most recent \(L\) events stored in the observation buffer,
striking a balance between capturing long-range historical dependence and maintaining computational efficiency. 
The $C\times p$-dimensional filter outputs $h$ are
concatenated with the   variables  $(S_t, q_t, T - t)$   and fed into an MLP.
The resulting outputs are then mapped to the admissible spread range via a sigmoid activation. 
The   output, denoted by $\mu_\theta(x)$, is  a $C$-dimensional vector representing the spread means based on observation    $x$, with one bid or ask spread associated with each asset.  
To facilitate exploration during policy learning, we embed $\mu_\theta$ into a Gaussian policy: 
\begin{equation*}
  \pi_\theta(d \delta \mid x) = \mathcal{N}\big(\delta;\, \mu_\theta(x),\,
     e^{2 \sigma_a} I_C \big) d \delta,
\end{equation*}
where $\sigma_a\in \mathbb R$ is a learnable log-standard-deviation parameter controlling the variance of the action distribution.

\subsection{Advantages of HSN}
\label{sec:advantage}

 We emphasize that 
 the exponential filter parameters  
 $(\alpha^{c,l}_j, \beta^{c,l}_j)$
 in the HSN 
are trained end-to-end (E2E) together with all other trainable parameters in the policy and value networks using model-free RL algorithms (we use PPO~\cite{schulman2017proximal}) to maximize the reward objective \eqref{eq:objective}.
Consequently, HSN learns a task-oriented Markovian approximation directly from interactions with the underlying non-Markovian Hawkes process, offering improved    efficiency  and interpretability, as discussed below.
 
\begin{itemize}[wide] 
    \item \textbf{Goal-oriented E2E training.} Our E2E training approach directly optimizes the   objective, with historical (non-Markovian) order-flow information summarized through learnable exponentially weighted features. This follows the goal-oriented paradigm in modern AI-driven finance (see, e.g.,  \cite{cong2021alphaportfolio,zhang2021universal,agal2025machine}), where representation learning and downstream decision-making are optimized jointly rather than treated as separate stages.   
In contrast, classical approaches require a multi-stage pipeline \cite{cartea_jaimungal_ricci_2014, martin_tse_2026}: estimating a non-Markovian Hawkes model from noisy observations, constructing a finite-dimensional Markovian approximation from the estimated kernel, and solving the resulting high-dimensional control problem. Our approach bypasses this model-estimation-then-optimization procedure by learning the relevant Hawkes representations jointly with the decision policy.

\item \textbf{Efficient representation learning for  order flows.} The feature representation \(h\) in \eqref{eq:policy_exponential} provides an efficient representation of the Hawkes-driven environment. 
The exponential filter compresses the historical $L$-order flow  events into a \(Cp\)-dimensional feature vector with $p\ll L$, compared with \(CL\)-dimensional   features arising from  a  frame stacking of   the raw event history or an equivalent hidden state representation learned by a recurrent architecture. 
Consequently, the MLP trunk receives a structured low-dimensional representation rather than having to infer the underlying Hawkes state directly from the raw event stream. This reduces the effective search space for   RL algorithms and improves sample efficiency.

Furthermore, the learning process can be accelerated by initializing the learnable parameters \((\alpha_j^{c,l},\beta_j^{c,l})\) using domain knowledge of financial event dynamics. 
For example, the decay parameters \(\beta\) can be initialized to cover a range of relevant financial timescales. 
For the two-mode variant used in Section~\ref{sec:experiments}, we initialize \(\beta\) to capture both fast and slow order-flow dynamics. 
The amplitudes \(\alpha\) can be initialized with small positive values such that the initial branching ratio
$ 
\sum_j {\alpha_j}/{\beta_j}
$ 
remains well below the subcriticality threshold.
Positivity of $(\alpha_j^{c,l},\beta_j^{c,l})$ can be easily  enforced by a
softplus reparameterization with a suitable numerical floor.

\item \textbf{Interpretable and reliable policy learning.}
The architecture also improves decision interpretability and admits theoretical support, as illustrated in Section~\ref{label:insight}. In the HSN architecture, the MLP trunk can be interpreted as an approximation of the feedback policy map associated with the market-making problem under the learned mixture-of-$p$-exponentials kernel representation. Consequently, for sufficiently large $p$ and a sufficiently expressive MLP, the HSN policy class approximates arbitrarily well the greedy policies induced by the mixture-of-$p$-exponential approximation \eqref{eq:kernel_approx_L_p}  of the true Hawkes kernel. Therefore, the policy learned by model-free RL can achieve performance at least comparable to these greedy policies, which, according to Theorem~\ref{thm:disc_trunc_kernel_reward_bnd}, become nearly optimal
as $p$ increases for a large family of kernels. 
 \end{itemize}

\section{Experiments}
\label{sec:experiments}

We evaluate HSN against six parameter-matched neural baselines across seven configurations in four settings: single-asset exponential (two branching ratios), multi-asset with cross-excitation, single-asset shifted power-law (two tail indices), and nonstationary regime-switching (two switch rates). On the five intensity-sensitive configurations, HSN attains the highest mean cumulative reward by $+5$--$32\%$, with the largest gain in the fractional-kernel regime.

\subsection{Methodology}
\label{subsec:methodology}

\paragraph{Environment, training, and evaluation.}

We consider the market-making problem introduced in Section~\ref{sec:problem} and its multi-asset extensions under the following model parameters:  $\sigma = 0.1$, $S_0 = 100$,
$\epsilon^\pm  = 0$ (a pure diffusion midprice, so that executed
trades do not directly move the midprice),
initial inventory $q_0 = 0$, running inventory penalty
$\gamma = 0.10$, terminal-inventory liquidation coefficient
$\kappa  = 0.001$,
inventory cap $|Q| \le 5$,
and history buffer of the last $L=32$
event timestamps. The episode length is $T = 0.5$ for the single-asset
exponential, multi-asset, and fractional settings, and $T = 1.0$ for the
regime-switching setting (so that the parameters experience sufficient regime switches within each episode). The agent's action at each event is a
pair of continuous bid and ask spreads
$(\delta^+, \delta^-) \in [0.005, 0.30]^2$ drawn from a Gaussian policy
(Section~\ref{sec:hsn}). The fill probabilities follow
$p^\pm(\delta^\pm) = \exp(-10 \delta^\pm)$ as in \cite{cartea_jaimungal_ricci_2014}. Each (architecture,
environment) cell is trained with PPO~\cite{schulman2017proximal}
(clipping $\epsilon = 0.2$, GAE~\cite{schulman_2016_gae} $\lambda =
0.95$, $\gamma = 0.99$, entropy coef.\ $0.003$, value-loss coef.\ $0.5$,
Adam optimizer~\cite{kingma_2015_adam} with learning rate
$3{\times}10^{-4}$, 100 updates of 1024 rollout steps, minibatch size
256, gradient clipping at norm $0.5$). After training, we
evaluate the greedy policy over 100 episodes
with fixed evaluation seeds and report the mean cumulative episodic
reward averaged over 30 random seeds (mean $\pm$ SEM).
Pairwise Welch's $t$-test against baselines give $p$-values.

\paragraph{Architectures.}
HSN uses $p = 2$ modes per channel. In the single-asset exponential configuration, the mixture filter is initialized kernel-agnostically at $(\beta_1, \beta_2) = (6.0, 0.4)$ and $(\alpha_1, \alpha_2) = (1.5, 0.1)$, an initial branching ratio of $0.5$ with neither decay rate coinciding with a true kernel mode. The remaining configurations initialize at $(\beta_1, \beta_2) = (5.0, 0.3)$ and $(\alpha_1, \alpha_2) = (0.2, 0.4)$. The filter parameters $(\alpha_j^c, \beta_j^c)$ are trained E2E, so the initialization only sets the starting point. Architecture receives up to the last $L = 32$ event observations with side labels, timestamps, and the current scalar state $(S, q, T-t)$ (single-asset) or $(S^A, q^A, S^B, q^B, T-t)$ (multi-asset). Extracted features and scalars are fed into a common two-hidden-layer MLP head with the continuous Gaussian policy of Section~\ref{sec:hsn}, so the only architectural variable is the feature extractor. Trainable-parameter counts are matched within a common budget across architectures. The baselines are: \textbf{MLP+frame-stack} (raw stack of the last 8 events, features $(\Delta t, \text{side})$; head hidden 64; $\approx 2.8$k params); \textbf{MLP+EMA} (eight fixed exponential moving averages of event deltas with $\beta_{\text{EMA}} \in \{0.1, 0.3, 0.7, 1.0, 2.0, 3.0, 5.0, 8.0\}$; head hidden 96; $\approx 2.6$k params); \textbf{GRU}~\cite{cho_2014_gru} (event embed $\R^2 \to \R^8$, hidden 16, head hidden 32; $\approx 2.4$k params); \textbf{LSTM}~\cite{hochreiter_1997_lstm} (same embed, hidden 14, head hidden 32; $\approx 2.3$k params); \textbf{Transformer}~\cite{vaswani_2017_transformer} (single-layer multi-head attention with $d_{\text{model}} = 16$, 2 heads, log-time positional features; head hidden 32; $\approx 2.6$k params); and \textbf{Simple MLP} (a no-memory baseline over the scalar state only; head hidden $240$ in single-asset and $160$ in multi-asset; $\approx 2.6$k params). Simple MLP is the minimum-information reference: any advantage of the other architectures over it is attributable to their use of the event history.

\paragraph{Environment configurations.}
Table~\ref{tab:results-summary} reports seven configurations grouped into four settings:
\begin{enumerate}[leftmargin=1.4em, labelindent=0pt, itemsep=2pt, topsep=2pt]
  \item \emph{Single-asset, mixture-of-exp. kernel} ($n \in \{0.75, 0.90\}$): two-channel Hawkes $(M^+, M^-)$ with common mixture-of-exponentials kernel $g(t) = \sum_{j=1}^{2} \alpha_j e^{-\beta_j t}$, $(\beta_1, \beta_2) = (5.0, 0.3)$, $\alpha_j = (n/2)\beta_j$, so $\|g\|_{L^1([0,\infty))} = n$; baseline intensities $\lambda_0^+ = 1.0$, $\lambda_0^- = 0.7$. The leading-order optimal spread of \cite{avellaneda_stoikov_2008, cartea_jaimungal_ricci_2014} depends only weakly on the intensity in this regime, so the theory predicts little intensity-representation advantage; we retain it as a control.
  \item \emph{Multi-asset, exponential kernel with cross-excitation} ($\mathit{cs} = 0.7$): four-channel Hawkes $M = (M^{A,+}, M^{A,-}, M^{B,+}, M^{B,-})$ on two assets $\{A, B\}$, with kernel matrix $g_{ij}(t) = \alpha_{ij}(\tfrac12 e^{-5.0\,t} + \tfrac12 e^{-0.3\,t})$; with off-diagonal entries coupling the same side across assets ($A^{+}\!\leftrightarrow\!B^{+}$ and $A^{-}\!\leftrightarrow\!B^{-}$), so the per-channel self- and cross-excitation branching are $n(1-\mathit{cs})/2 = 0.12$ and $n \cdot\mathit{cs}/2 = 0.28$, giving a spectral radius $\rho(M) = n/2 = 0.40$ (with input branching budget $n = 0.80$). Baseline intensities are symmetric across sides of each asset: $\lambda_0^{A,\pm} = 1.5$ and $\lambda_0^{B,\pm} = 1.2$; initial prices $S^A_0 = 100$, $S^B_0 = 50$. The agent posts four continuous spreads, the reward aggregates cash plus inventory penalties across assets.
  \item \emph{Single-asset, shifted power-law kernel} ($\varepsilon \in \{0.50, 0.15\}$): the two-channel setup of~(1) with kernel $g(t) = c(t + \tau_0)^{-(1+\varepsilon)}$, $\tau_0 = 1$, $c = n \varepsilon \tau_0^\varepsilon$, so $\|g\|_{L^1([0,\infty))} = n$, $n = 0.7$, and $\varepsilon = 0.15$ matching the empirical Hardiman--Bercot--Bouchaud regime \cite{hardiman_2013}. The kernel admits no exact finite mixture-of-exponentials representation, stressing the kernel-approximation axis of Theorem~\ref{thm:disc_trunc_kernel_reward_bnd}.
  \item \emph{Nonstationary, regime-switching kernel} ($\mathit{sr} \in \{2.0, 4.0\}$): the two-channel setup of~(1) with kernel parameters following a continuous-time Markov chain over two states with branching ratios $(n_{\text{low}}, n_{\text{high}}) = (0.50, 0.85)$ and symmetric switch rate $\mathit{sr}$ per unit time. Nonstationarity acts through the kernel parameters only, not through jumps in the mid-price; the regime is unobserved and must be inferred from the event history.
\end{enumerate}

\subsection{Results}

The empirical results are organized along three axes: downstream reward under each environment (Table~\ref{tab:results-summary} and Figure~\ref{fig:seed-distributions}), the mechanism connecting HSN's feature basis to the intensity representation (Table~\ref{tab:intensity-mse}, Section~\ref{subsec:intensity}), and the quantitative kernel-approximation rate underlying the regret guarantee (Figure~\ref{fig:psweep}, Section~\ref{subsec:decay-validation}). 

\subsubsection{Trading performance}
Table~\ref{tab:results-summary} reports the mean cumulative reward $\hat V$ and the HSN-vs-baseline gap for each (architecture, environment) combination, with columns grouped by setting.
HSN attains the highest mean cumulative reward on the five intensity-sensitive configurations (columns 3--7 of Table~\ref{tab:results-summary}), with $p < 0.01$ to $p < 0.001$ across 30 seeds each. The largest advantages occur in the heavy-tailed fractional regime ($+32.5\%$ at $\varepsilon = 0.15$), the multi-asset cross-excitation setting ($+16.6\%$ at $\mathit{cs} = 0.7$), and the regime-switching setting ($+11.1\%$ at $\mathit{sr} = 2$). The no-memory Simple MLP baseline underperforms HSN by $36$--$57\%$ on the same five columns, isolating HSN's advantage to its use of the event history rather than to parameter count or optimization. On the two single-asset exponential controls, HSN is statistically tied with the strongest baseline, MLP+frame-stack ($0.0171$ vs $0.0184$ at $n = 0.75$ and $0.0208$ vs $0.0211$ at $n = 0.90$, both $p > 0.05$), while outperforming every other architecture (including the no-memory Simple MLP) by $45$--$90\%$. This is consistent with the near-intensity-independence of the optimal spread in that regime: the intensity representation carries little policy-value signal, so HSN matches the best general-purpose baseline rather than exceeding it. Figure~\ref{fig:seed-distributions} shows the intensity-sensitive advantage is robust across seeds: HSN's lower whisker exceeds the upper whiskers of Transformer, LSTM, MLP+EMA, and Simple MLP in both the cross-excitation and fractional configurations.

\begin{table*}[h]
  \centering
  \caption{Mean cumulative reward $\hat V$ across 30 seeds (continuous
  actions, parameter-matched architectures). Subscripts denote SEM.
  Bold entries are the column maximum. Bottom row: relative gap of HSN
  vs the strongest competing baseline in each column, shown where the
  difference is statistically significant (*** $p<0.001$, ** $p<0.01$,
  Welch's $t$-test); ``tie'' denotes no significant difference ($p>0.05$). Column groups: single-asset with
  exponential Hawkes kernel (branching ratio $n$); multi-asset with
  exponential Hawkes kernel and cross-asset excitation (cross-strength
  $\mathit{cs}$); single-asset with fractional Hawkes
  kernel (tail index $\varepsilon$); single-asset with regime-switching
  exponential Hawkes kernel (switch rate $\mathit{sr}$).}
  \vspace{-1mm}
  \label{tab:results-summary}
  \setlength{\tabcolsep}{4pt}
  \small
  \begin{tabular}{lccccccc}
    \toprule
    & \multicolumn{2}{c}{Single-asset, mix-of-exp.}
    & Multi-asset, exp.
    & \multicolumn{2}{c}{Single-asset, fractional}
    & \multicolumn{2}{c}{Regime-switching, mix-of-exp.} \\
    \cmidrule(lr){2-3} \cmidrule(lr){4-4} \cmidrule(lr){5-6} \cmidrule(lr){7-8}
    Architecture & $n{=}0.75$ & $n{=}0.90$ & $\mathit{cs}{=}0.7$ & $\varepsilon{=}0.50$ & $\varepsilon{=}0.15$ & $\mathit{sr}{=}2$ & $\mathit{sr}{=}4$ \\
    \midrule
    \textbf{HSN}    & $\mathbf{0.0171}_{{\pm.0007}}$ & $\mathbf{0.0208}_{{\pm.0004}}$ & $\mathbf{0.0548}_{{\pm.0007}}$ & $\mathbf{0.0217}_{{\pm.0002}}$ & $\mathbf{0.0216}_{{\pm.0004}}$ & $\mathbf{0.0561}_{{\pm.0009}}$ & $\mathbf{0.0451}_{{\pm.0006}}$ \\
    Simple MLP      & $0.0090_{{\pm.0000}}$ & $0.0144_{{\pm.0000}}$ & $0.0349_{{\pm.0003}}$ & $0.0125_{{\pm.0000}}$ & $0.0093_{{\pm.0000}}$ & $0.0339_{{\pm.0002}}$ & $0.0286_{{\pm.0001}}$ \\
    MLP+EMA         & $0.0110_{{\pm.0001}}$ & $0.0177_{{\pm.0001}}$ & $0.0349_{{\pm.0003}}$ & $0.0163_{{\pm.0003}}$ & $0.0127_{{\pm.0003}}$ & $0.0389_{{\pm.0003}}$ & $0.0342_{{\pm.0003}}$ \\
    MLP+frame-stack & $\mathbf{0.0184}_{{\pm.0006}}$ & $\mathbf{0.0211}_{{\pm.0004}}$ & $0.0470_{{\pm.0006}}$ & $0.0206_{{\pm.0003}}$ & $0.0163_{{\pm.0002}}$ & $0.0505_{{\pm.0004}}$ & $0.0428_{{\pm.0005}}$ \\
    GRU             & $0.0094_{{\pm.0001}}$ & $0.0155_{{\pm.0002}}$ & $0.0450_{{\pm.0010}}$ & $0.0168_{{\pm.0002}}$ & $0.0148_{{\pm.0001}}$ & $0.0443_{{\pm.0005}}$ & $0.0338_{{\pm.0004}}$ \\
    LSTM            & $0.0098_{{\pm.0001}}$ & $0.0162_{{\pm.0003}}$ & $0.0357_{{\pm.0006}}$ & $0.0135_{{\pm.0001}}$ & $0.0102_{{\pm.0001}}$ & $0.0375_{{\pm.0005}}$ & $0.0348_{{\pm.0005}}$ \\
    Transformer     & $0.0109_{{\pm.0002}}$ & $0.0166_{{\pm.0002}}$ & $0.0359_{{\pm.0004}}$ & $0.0141_{{\pm.0001}}$ & $0.0105_{{\pm.0001}}$ & $0.0412_{{\pm.0005}}$ & $0.0354_{{\pm.0005}}$ \\
    \midrule
    HSN vs.\ best baseline & tie & tie & $+16.6\%$*** & $+5.3\%$** & $+32.5\%$*** & $+11.1\%$*** & $+5.4\%$** \\
    \bottomrule
  \end{tabular}
  \vspace{-2mm}
\end{table*}


\begin{figure*}[t]
  \centering
  \begin{subfigure}[b]{0.48\textwidth}
    \centering
    \begin{tikzpicture}
      \begin{axis}[
        width=\linewidth, height=3.0cm,
        boxplot/draw direction=y,
        ylabel={$\hat V$ (30 seeds)},
        ylabel near ticks,
        ymin=0.028, ymax=0.070,
        ymajorgrids=true, grid style={dashed, gray!30},
        xtick={1,2,3,4,5,6,7},
        xticklabels={HSN,Simple MLP,MLP+EMA,MLP+FS,GRU,LSTM,Trans},
        xticklabel style={rotate=35, anchor=east, font=\scriptsize},
        yticklabel style={font=\scriptsize},
        ytick={0.03,0.04,0.05,0.06,0.07},
        boxplot/every box/.style={ultra thin},
        boxplot/every whisker/.style={thin},
        boxplot/every median/.style={thick, black},
      ]
        \addplot+[boxplot prepared={lower whisker=0.04569, lower quartile=0.05237,
          median=0.05415, upper quartile=0.05750, upper whisker=0.06665},
          fill=blue!25, draw=blue!70!black] coordinates {};
        \addplot+[boxplot prepared={lower whisker=0.03197, lower quartile=0.03389,
          median=0.03472, upper quartile=0.03602, upper whisker=0.03802},
          fill=yellow!25, draw=olive!70!black] coordinates {};
        \addplot+[boxplot prepared={lower whisker=0.03210, lower quartile=0.03380,
          median=0.03467, upper quartile=0.03578, upper whisker=0.03879},
          fill=orange!25, draw=orange!80!black] coordinates {};
        \addplot+[boxplot prepared={lower whisker=0.04124, lower quartile=0.04449,
          median=0.04702, upper quartile=0.04853, upper whisker=0.05418},
          fill=teal!25, draw=teal] coordinates {};
        \addplot+[boxplot prepared={lower whisker=0.03618, lower quartile=0.04098,
          median=0.04326, upper quartile=0.04988, upper whisker=0.05460},
          fill=red!20, draw=red!60!black] coordinates {};
        \addplot+[boxplot prepared={lower whisker=0.03046, lower quartile=0.03288,
          median=0.03495, upper quartile=0.03694, upper whisker=0.04224},
          fill=purple!15, draw=purple!60!black] coordinates {};
        \addplot+[boxplot prepared={lower whisker=0.03184, lower quartile=0.03481,
          median=0.03587, upper quartile=0.03657, upper whisker=0.04401},
          fill=gray!15, draw=gray!70!black] coordinates {};
      \end{axis}
    \end{tikzpicture}
    \caption{Multi-asset cross-excitation ($cs=0.7$).}
    \label{fig:box-cs07}
  \end{subfigure}
  \hfill
  \begin{subfigure}[b]{0.48\textwidth}
    \centering
    \begin{tikzpicture}
      \begin{axis}[
        width=\linewidth, height=3.0cm,
        boxplot/draw direction=y,
        ylabel={$\hat V$ (30 seeds)},
        ylabel near ticks,
        ymin=0.008, ymax=0.028,
        ymajorgrids=true, grid style={dashed, gray!30},
        xtick={1,2,3,4,5,6,7},
        xticklabels={HSN,Simple MLP,MLP+EMA,MLP+FS,GRU,LSTM,Trans},
        xticklabel style={rotate=35, anchor=east, font=\scriptsize},
        yticklabel style={font=\scriptsize},
        ytick={0.010,0.014,0.018,0.022,0.026},
        boxplot/every box/.style={ultra thin},
        boxplot/every whisker/.style={thin},
        boxplot/every median/.style={thick, black},
      ]
        \addplot+[boxplot prepared={lower whisker=0.01852, lower quartile=0.01964,
          median=0.02176, upper quartile=0.02287, upper whisker=0.02633},
          fill=blue!25, draw=blue!70!black] coordinates {};
        \addplot+[boxplot prepared={lower whisker=0.00906, lower quartile=0.00919,
          median=0.00923, upper quartile=0.00930, upper whisker=0.00947},
          fill=yellow!25, draw=olive!70!black] coordinates {};
        \addplot+[boxplot prepared={lower whisker=0.01058, lower quartile=0.01134,
          median=0.01168, upper quartile=0.01411, upper whisker=0.01662},
          fill=orange!25, draw=orange!80!black] coordinates {};
        \addplot+[boxplot prepared={lower whisker=0.01295, lower quartile=0.01585,
          median=0.01659, upper quartile=0.01703, upper whisker=0.01843},
          fill=teal!25, draw=teal] coordinates {};
        \addplot+[boxplot prepared={lower whisker=0.01372, lower quartile=0.01435,
          median=0.01483, upper quartile=0.01510, upper whisker=0.01641},
          fill=red!20, draw=red!60!black] coordinates {};
        \addplot+[boxplot prepared={lower whisker=0.00950, lower quartile=0.00986,
          median=0.00994, upper quartile=0.01029, upper whisker=0.01223},
          fill=purple!15, draw=purple!60!black] coordinates {};
        \addplot+[boxplot prepared={lower whisker=0.00942, lower quartile=0.01019,
          median=0.01035, upper quartile=0.01080, upper whisker=0.01417},
          fill=gray!15, draw=gray!70!black] coordinates {};
      \end{axis}
    \end{tikzpicture}
    \caption{Fractional kernel ($\varepsilon=0.15$).}
    \label{fig:box-frac}
  \end{subfigure}

  \caption{Distribution of final $\hat V$ across 30 seeds for two
    intensity-sensitive configurations of Table~\ref{tab:results-summary}.
    Boxes span the interquartile range (median line) and whiskers show
    min/max. Architectures are ordered by feature-extractor complexity:
    HSN, Simple MLP (no history), MLP+EMA, MLP+frame-stack (raw stack),
    GRU, LSTM, Transformer. HSN's entire distribution lies above almost all
    medians of all six baselines in both environments. }
  \label{fig:seed-distributions}
  \vspace{-2mm}
\end{figure*}
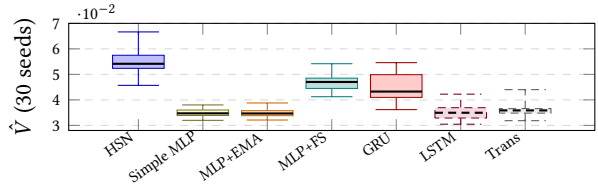
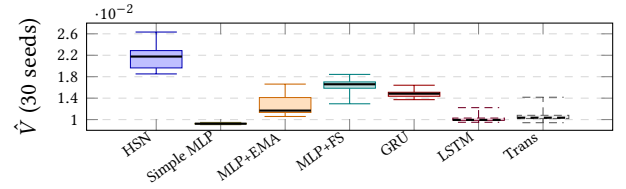

\subsubsection{Intensity representation}
\label{subsec:intensity}

We show that the HSN   provides a   more effective representation of the Hawkes intensity than  baseline architectures.   
Each architecture is trained with PPO for $100$ updates over $10$ seeds in a single-asset reduced-form Hawkes-MM environment at $n = 0.95$, under two kernels: a single exponential with $\beta = 1.5$ and a mixture-of-three-exponentials with $\beta = (0.2, 1.0, 5.0)$ and $\alpha_j = (n/3)\beta_j$. The higher branching ratio (versus $0.75$--$0.90$ in Table~\ref{tab:results-summary}) yields a more discriminating comparison, as it places the intensity in a regime with larger variance about its stationary mean. Other environment and PPO settings match Section~\ref{subsec:methodology}. HSN uses $p = 3$ modes with initialization $\beta = (1.0, 2.0, 5.0)$ and $\alpha = (0.2, 0.4, 0.5)$.

\paragraph{Linear-probe protocol.}
For each frozen model we generate $2{,}000$ samples of the form $(\text{features}, \lambda^+, \lambda^-)$ by rolling out the environment under random actions with a probe seed disjoint from the training seeds. Features are the architecture's penultimate vector: the concatenation of the current scalar state $(S, q, T-t)$ with the architecture-specific memory representation. Ground-truth intensity values are read directly from the simulator. We split $1{,}600 / 400$ into fit and held-out sets, fit ordinary least-squares from features to $(\lambda^+, \lambda^-)$ on the fit set, and report the MSE on the held-out set. For HSN we additionally report the zero-parameter direct estimate $\hat\lambda^c = \lambda_0^c + \sum_j h^c_j$ on the same held-out samples, as a strictly-conservative reference.

\begin{table}[t]
  \small
  \centering
  \caption{Intensity MSE on a $400$-sample held-out set under the
  linear-probe protocol ($n=0.95$, single-asset, 10 seeds; SEM in
  subscript). The first row (``HSN, direct'') reports HSN's
  zero-fitted-parameter direct estimate
  $\hat\lambda^c = \lambda_0^c + \sum_j h^c_j$ evaluated on the held-out
  samples with \emph{no} OLS fit; every other row uses the same OLS
  linear probe fit from that architecture's feature vector to the
  ground-truth $(\lambda^+, \lambda^-)$ on 1{,}600 training samples
  (indicated by the ``$+$ OLS'' suffix).}
  \label{tab:intensity-mse}
  \begin{tabular}{lcc}
    \toprule
    Architecture & Exp.\ ($\beta{=}1.5$) & Mix-of-exp.\ (3-mode) \\
    \midrule
    HSN, direct (no OLS)   & $0.20_{{\pm.01}}$ & $0.53_{{\pm.02}}$ \\
    \textbf{HSN + OLS}     & $\mathbf{<10^{-6}}$ & $\mathbf{<10^{-6}}$ \\
    Simple MLP + OLS       & $0.72_{{\pm.03}}$ & $1.09_{{\pm.03}}$ \\
    MLP+EMA + OLS          & $0.38_{{\pm.01}}$ & $0.60_{{\pm.02}}$ \\
    MLP+frame-stack + OLS  & $0.26_{{\pm.04}}$ & $0.43_{{\pm.03}}$ \\
    GRU + OLS              & $1.00_{{\pm.35}}$ & $1.39_{{\pm.45}}$ \\
    LSTM + OLS             & $0.21_{{\pm.06}}$ & $0.31_{{\pm.06}}$ \\
    Transformer + OLS      & $0.32_{{\pm.02}}$ & $0.40_{{\pm.03}}$ \\
    \bottomrule
  \end{tabular}
  \vspace{-7mm}
\end{table}

Two observations follow from Table~\ref{tab:intensity-mse}. HSN's zero-fitted-parameter direct estimate is competitive with the strongest OLS-refitted baseline on the single-exponential kernel ($0.20$ vs LSTM $0.21$ within one SEM) and mid-pack on the mixture-of-three-exponentials kernel. The residual amplitude-scaling error reflects that PPO tunes $\alpha_j^c$ against the reward objective \eqref{eq:objective}, not the intensity MSE. Under the OLS refit that every baseline receives, HSN's held-out MSE drops by 6-7 orders of magnitude below baselines, which indicates that the decay rates $\beta_j^c$ implicit in \eqref{eq:policy_exponential} accurately capture the true kernel.


\begin{figure*}[t]
  \centering
  \begin{subfigure}[t]{0.48\textwidth}
    \centering
    \begin{tikzpicture}
      \begin{semilogyaxis}[
        width=\linewidth, height=2.99cm,
        xlabel={number of exponential modes $p$},
        ylabel={$L^2$ kernel error},
        ylabel near ticks,
        xmin=0.5, xmax=8.5,
        ymin=1e-5, ymax=1,
        ymajorgrids=true, grid style={dashed, gray!30},
        xtick={1,2,3,4,6,8},
        legend style={font=\tiny, at={(0.5,1.02)}, anchor=south,
                       legend columns=-1, /tikz/every even column/.append style={column sep=0.5em},
                       fill=white, draw=gray!40},
        legend cell align=left,
      ]
        \addplot[blue!70!black, thick, mark=*, mark size=2pt]
          coordinates {(1,2.57e-3)(2,1.79e-3)(3,2.10e-3)(4,6.5e-4)(6,2.1e-4)(8,2e-5)};
        \addlegendentry{exp1}
        \addplot[teal, thick, mark=square*, mark size=2pt]
          coordinates {(1,0.281)(2,0.155)(3,0.054)(4,0.017)(6,2e-4)(8,1.5e-4)};
        \addlegendentry{sumexp3}
        \addplot[red!70!black, thick, mark=triangle*, mark size=2.5pt]
          coordinates {(1,0.177)(2,0.121)(3,0.061)(4,0.046)(6,0.044)(8,0.044)};
        \addlegendentry{frac015}
      \end{semilogyaxis}
    \end{tikzpicture}
    \caption{$L^2$ kernel-approximation error
      $\|g - \sum_{j=1}^p \alpha_j e^{-\beta_j t}\|_{L^2([0,T])}$, using NNLS-optimal amplitudes on
      HSN's learned decay rates.}
    \label{fig:psweep-kernel-l2}
  \end{subfigure}
  \hfill
  \begin{subfigure}[t]{0.48\textwidth}
    \centering
    \begin{tikzpicture}
      \begin{semilogyaxis}[
        width=\linewidth, height=2.99cm,
        xlabel={number of exponential modes $p$},
        ylabel={intensity MSE},
        ylabel near ticks,
        xmin=0.5, xmax=8.5,
        ymin=1e-6, ymax=10,
        ymajorgrids=true, grid style={dashed, gray!30},
        xtick={1,2,3,4,6,8},
        legend style={font=\tiny, at={(0.5,1.02)}, anchor=south,
                       legend columns=-1, /tikz/every even column/.append style={column sep=0.5em},
                       fill=white, draw=gray!40},
        legend cell align=left,
      ]
        \addplot[blue!70!black, thick, mark=*, mark size=2pt]
          coordinates {(1,1.7e-4)(2,9e-5)(3,1.1e-4)(4,1e-5)(6,5e-6)(8,5e-6)};
        \addlegendentry{exp1}
        \addplot[teal, thick, mark=square*, mark size=2pt]
          coordinates {(1,0.7276)(2,0.2362)(3,0.0367)(4,0.00130)(6,5e-6)(8,5e-6)};
        \addlegendentry{sumexp3}
        \addplot[red!70!black, thick, mark=triangle*, mark size=2.5pt]
          coordinates {(1,0.2130)(2,0.0972)(3,0.0229)(4,0.0140)(6,0.0122)(8,0.0122)};
        \addlegendentry{frac015}
      \end{semilogyaxis}
    \end{tikzpicture}
    \caption{Best-fit intensity MSE on a $2{,}000$-sample held-out set,
      obtained by 
      regressing   the true intensity onto HSN's
      $p$-mode filter outputs.}
    \label{fig:psweep-intensity}
  \end{subfigure}
  \caption{HSN $p$-sweep on three single-asset kernels ($n = 0.85$,
    $T = 10$; 20 seeds per cell).
    (a) The $L^2$ kernel-approximation error; (b) the intensity-representation error. Both quantities decrease with $p$ and
    track each other across the three kernel classes.}
  \label{fig:psweep}
\end{figure*}
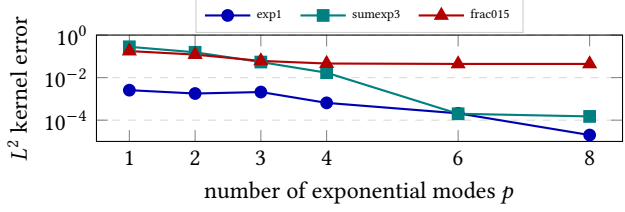
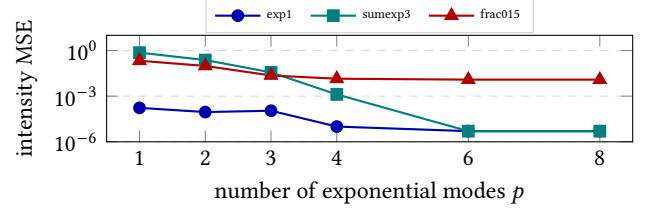

\subsubsection{Implicit kernel learning}
\label{subsec:decay-validation}

The trained HSN captures the excitation-kernel structure of the underlying Hawkes process even without an explicit kernel-approximation loss. We train HSN with $p \in \{1, 2, 3, 4, 6, 8\}$ modes at $n = 0.85$ on three single-asset kernels: a single exponential (exp1) with $\beta = 1.0$, a sum-of-three-exponentials with $\beta = (0.2, 1.0, 5.0)$ (sumexp3), and a shifted power-law with $\varepsilon = 0.15$, $\tau_0 = 0.1$ (frac015). The horizon $T = 10$ lets the long-memory intensity build up, and $\tau_0 = 0.1$ makes the power-law kernel less regular near $t = 0$. For each $p$, the initial decay rates $\beta_j$ are log-spaced over $[0.1, 10]$ ($\beta = 1$ at $p = 1$) with $\alpha_j = (n/p)\beta_j$. The observation buffer $L = 32$ is unchanged.

After training we freeze each network and measure two quantities on a $2{,}000$-sample held-out set (Figure~\ref{fig:psweep}): the kernel approximation error $\|g - \sum_{j} \alpha_j e^{-\beta_j t}\|_{L^2([0,T])}$ and the intensity MSE against $\lambda^c(t)$. Both use amplitudes $\alpha_j$ obtained by non-negative least squares (NNLS) on the frozen $\beta_j$ from PPO, with $\alpha_j \ge 0$ preserving the Bernstein representation. This refit plays the same role as the OLS linear probe of Section~\ref{subsec:intensity}: it isolates the capacity of the discovered feature basis from the amplitude scaling that PPO's reward loss induces.

Three patterns are visible in Figure~\ref{fig:psweep}. For exp1 both errors are already small at $p = 1$, since the true kernel is a single exponential. For sumexp3 the errors drop sharply and floor at $p \ge 6$, once HSN's modes cover the three true decay rates. For frac015 the errors decay initially and plateau over $p \in \{4, 6, 8\}$ at a larger residual than in the exponential cases, consistent with the impossibility of representing a fractional kernel exactly by a finite sum of exponentials. Across all three kernels, the kernel-approximation error and the intensity-representation error track each other qualitatively.

\section{Proof of the Regret Bound}
\label{sec:proof}
We proceed to prove Theorem \ref{thm:disc_trunc_kernel_reward_bnd}.
 We first bound $\lvert J^g(\pi) - J^{\Hat{g}}(\pi) \rvert$  in terms of the kernel error $g-\hat g$, following arguments similar to \cite{martin_tse_2026}.

\begin{lemma}
    \label{lemma:PnL_gap}
    For each $c \in \{+,-\}$, let $g^c, \Hat{g}^c \in  L^2([0,T],  \mathbb{R}^{\ge 0})$ with $\lVert g^c \rVert_{L^1([0,T])} < 1$ and $\lVert \Hat{g}^c \rVert_{L^1([0,T])} < 1$. Then
    \begin{equation*}
        \sup_{\pi \in \Pi}\lvert J^g(\pi) - J^{\Hat{g}}(\pi) \rvert \le \hat {C} \left( \sum_{c \in \{+,-\}} \Hat{H}^c \lVert g^c - \Hat{g}^c \rVert_{L^2([0,T])} \right)^{1/2},
    \end{equation*}
    where $\Hat{H}^c=\frac{\Lambda_c T}{(1 - \lVert g^c \rVert_{L^1([0,T])}) (1 - \lVert \Hat{g}^c \rVert_{L^1([0,T])})} $, and $\hat{C} \coloneqq 2 \sqrt{2} \|\Hat{R}^*\|_{L^2} T^{1/4}$ with  $\Hat{R}^*$ given in \eqref{eq:Rstar}.
\end{lemma}

\begin{proof}
    Let $(\Omega, (\mathcal{F}_t)_{t \in [0,T]}, \mathcal{F}, \mathbb{P})$ be the filtered probability space supporting the one-dimensional   Brownian motion $W$, and two Poisson random measures $\mathcal{N}^+$ and $\mathcal{N}^-$ on $[0,T] \times [0,\infty)$, each with  a Lebesgue intensity measure. We assume $W$, $\mathcal{N}^+$, and $\mathcal{N}^-$ are mutually independent.
    On this probability space, we construct Hawkes processes as in \cite[Appendix A]{martin_tse_2026} using Poisson random measures embeddings (see, e.g., \cite{bremaud1996stability, hillairet2023expansion}). Let $c \in \{+, -\}$.
    Define $M_t^c \coloneqq \int_{(0,t]} \int_0^\infty \mathds{1}_{[u \le \lambda_{s-}^c]} \, \mathcal{N}^c(ds, du)$ and $\Hat{M}_t^c \coloneqq \int_{(0,t]} \int_0^\infty \mathds{1}_{[u \le \Hat{\lambda}_{s-}^c]} \, \mathcal{N}^c(ds, du)$, 
    where $\lambda_t^c \coloneqq \lambda_0^c(t) + \int_{[0,t)} g^c(t-s) \, d M_s^c$ and $\Hat{\lambda}_t^c \coloneqq \lambda_0^c(t) + \int_{[0,t)} \Hat{g}^c(t-s) \, d \Hat{M}_s^c$.
    Then $M^c$ and $\Hat{M}^c$ define counting processes with intensity $\lambda^c$ and $\Hat{\lambda}^c$, respectively. Similarly, define the execution processes $N_t^c$ and $\Hat{N}_t^c$ by
    $N_t^c \coloneqq \int_{(0,t]} \int_0^\infty \mathds{1}_{[u \le \lambda_{s-}^c p_s^c]} \, \mathcal{N}^c(ds, du)$
    and $\Hat{N}_t^c \coloneqq \int_{(0,t]} \int_0^\infty \mathds{1}_{[u \le \Hat{\lambda}_{s-}^c \Hat{p}_s^c]} \, \mathcal{N}^c(ds, du)$.
    Define the mismatches of order arrival   $D_M^c$ and   execution   $D_N^c$ by
    \begin{equation*}
        \begin{split}
            &D_M^c(t) \coloneqq \int_0^t \int_0^\infty \left\lvert \mathds{1}_{[u \le \lambda_{s-}^c]} - \mathds{1}_{[u \le \Hat{\lambda}_{s-}^c]} \right\rvert \, \mathcal{N}^c(ds, du), \\
            &D_N^c(t) \coloneqq \int_0^t \int_0^\infty \left\lvert \mathds{1}_{[u \le \lambda_{s-}^c p_s^c]} - \mathds{1}_{[u \le \Hat{\lambda}_{s-}^c \Hat{p}_s^c]} \right\rvert \, \mathcal{N}^c(ds, du).
        \end{split}
    \end{equation*}
    Let $\tau \coloneqq \inf{\{ t \ge 0 \, \vert \, D_M^c(t) > 0 \text{ or } D_N^c(t) > 0 \; \text{for some } c \in \{+,-\} \}}$. For $t < \tau$, $p_t^c = \Hat{p}_t^c$. Further, on $A \coloneqq \{   \tau  \le T \}$, $D_M^c(\min(\tau, T)) > 0$ or $D_N^c(\min(\tau, T)) > 0$ for some $c \in \{+,-\}$. Hence,
    \begin{equation*}
        \begin{split}
            \mathbb{P}[A] &\le \sum_{c \in \{+,-\}} \left( \mathbb{E}[D_M^c(\min(\tau, T))] + \mathbb{E}[D_N^c(\min(\tau, T))] \right) \\
            &= \sum_{c \in \{+,-\}} \mathbb{E}\left[ \int_0^{\min(\tau, T)}  \left( \lvert \lambda_t^c - \Hat{\lambda}_t^c \rvert + p_t^c \lvert \lambda_t^c - \Hat{\lambda}_t^c \rvert \right) \, dt \right] \\
            &\le 2 \sum_{c \in \{+,-\}} \mathbb{E}\left[ \int_0^{\min(\tau, T)} \left\lvert \lambda_t^c - \Hat{\lambda}_t^c \right\rvert \, dt \right] \le 2 \sum_{c \in \{+,-\}} \mathbb{E}[D_M^c(T)],
        \end{split}
    \end{equation*}
    where the equality holds by the compensator identity, and the second inequality by $p_t^c \le 1$.
    On the other hand, on $A^\complement$, $M_t^c = \Hat{M}_t^c$ and $N_t^c = \Hat{N}_t^c$ for all $t \in [0,T]$. As the Brownian motion is common, $S_t = \Hat{S}_t$. Consequently, $\delta_t^c = \Hat{\delta}_t^c$, $X_t = \Hat{X}_t$, and $Q_t = \Hat{Q}_t$. Thus, $R_\pi^g = R_\pi^{\Hat{g}}$ on $A^\complement$. Define
    \begin{equation}
        \label{eq:Rstar}
        \Hat{R}^* \coloneqq \lvert x_0 \rvert + (S^* + \Bar{\delta}) C_{M,\Hat{M}} + Q^* S^* + (\kappa + \gamma T) (Q^*)^2,
    \end{equation}
    where $S^* \coloneqq \lvert S_0 \rvert + \sigma \sup_{t \in [0,T]} \lvert W_t \rvert + \max{(\epsilon^+, \epsilon^-)} C_{M,\Hat{M}}$,  $C_{M,\Hat{M}} = M_T^+ + M_T^- + \Hat{M}_T^+ + \Hat{M}_T^-$, and $Q^* \coloneqq \lvert q_0 \rvert + C_{M,\Hat{M}}$. Note that $\Hat{R}^*$ has finite moments since  $S^*$ has  finite  moments   (by Doob's maximal inequality, see \cite[Theorem A.10]{schilling2014brownian}) and $M_T^c$, $\Hat{M}_T^c$ have finite  moments   (see \cite{leblanc2024exponential}). Then $\lvert R_\pi^g \rvert \le \Hat{R}^*$ and $\lvert R_\pi^{\Hat{g}} \rvert \le \Hat{R}^*$ for any $\pi \in \Pi$. Thus, by the Cauchy-Schwarz inequality
    \begin{equation}
    \label{J_different_kernel}
    \lvert J^g(\pi) - J^{\Hat{g}}(\pi) \rvert \le 2 \|\Hat{R}^*\|_{L^2} (\mathbb{P}[A])^{1/2}.
    \end{equation}
    Similarly to \cite[Proposition A.1]{martin_tse_2026} and using $\mathbb{E}[\Hat{M}_T^c] \le T \Lambda_c / (1 - \lVert \Hat{g}^c \rVert_{L^1([0,T])})$, we can show that
    \begin{equation*}
        \mathbb{E}[D_M^c(T)] \le \frac{\Lambda_c T}{(1 - \lVert g^c \rVert_{L^1([0,T])}) (1 - \lVert \Hat{g}^c \rVert_{L^1([0,T])})} \lVert g^c - \Hat{g}^c \rVert_{L^1([0,T])}.
    \end{equation*}
  This with \eqref{J_different_kernel}
  and $ \lVert g^c - \Hat{g}^c \rVert_{L^1([0,T])}\le  \sqrt{T}\lVert g^c - \Hat{g}^c \rVert_{L^2([0,T])}$ yield the desired result.
\end{proof}

\begin{proof}[Proof of Theorem \ref{thm:disc_trunc_kernel_reward_bnd}]
   Since $ \lVert \Hat{g}^{\pm,K_{1}, K_{2}, p} \rVert_{L^1([0,T])} \le \rho^* < 1$,  $  \Hat{H}^{\pm}$ are finite,   the moments of the Hawkes processes associated with kernel $\Hat{g}^{\pm,K_{1}, K_{2}, p}$
   can be bounded uniformly in $\rho^*$, independent of $K_1,K_2$ and $p$  (see \cite{leblanc2024exponential}). Thus, by the definition  \eqref{eq:Rstar} of $R^{*}$  and Doob's maximal inequality, $ \|R^{*}\|_{L^2} < \infty$. Note  that $\sup_{\pi \in \Pi} J^g(\pi) - J^g(\hat \pi) \le 2 \sup_{\pi \in \Pi} \lvert J^g(\pi) - J^{\Hat{g}}(\pi) \rvert$. Combining the above with Proposition   \ref{prop:L2_approx} and Lemma \ref{lemma:PnL_gap}  yields the desired result.
\end{proof}

\section{Conclusion}
\label{sec:conclusion}

We proposed a model-free RL framework for market making under general Hawkes order-flow dynamics and quantified the regret induced by a learned Markovian approximation of the underlying non-Markovian dynamics. Guided by our theoretical analysis, we developed the Hawkes-Structured Network   to learn compact Hawkes-aware representations of historical order flow directly from reward feedback while simultaneously optimizing trading performance, without explicitly estimating excitation kernels or fill probabilities.  Extensive experiments demonstrate consistent improvements over general-purpose neural architectures across a wide range of Hawkes-driven market-making environments.


Several directions merit further investigation. First, in the multi-asset setting, the current HSN architecture models cross-asset interactions only implicitly through the MLP trunk, rather than explicitly encoding cross-channel Hawkes excitation. Introducing a $C \times C \times p$ cross-channel filter could directly embed the full multi-asset Hawkes structure into the feature extractor. 
More broadly, it remains an open question how task-oriented representation learning relates to the statistical recovery of underlying Hawkes processes.
Extending the proposed framework to richer market microstructure models and other non-Markovian stochastic control problems is also a promising direction for future research.

\vspace{-1mm}

\begin{acks}
Y.Z. is supported by a J.P.\ Morgan AI Faculty Research Award.
P.P. is supported by the Roth Scholarship (Imperial~College~London)
and the Excellence Scholarship (Gesellschaft f\"ur Forschungsf\"orderung
Nieder\"osterreich, a subsidiary of the Province of Lower Austria).
\end{acks}

\vspace{-1mm}

\bibliographystyle{ACM-Reference-Format}
\bibliography{references}


\end{document}